\documentclass{article}

\usepackage{arxiv}

\usepackage[utf8]{inputenc} % allow utf-8 input
\usepackage[T1]{fontenc}    % use 8-bit T1 fonts
\usepackage{hyperref}       % hyperlinks
\usepackage{url}            % simple URL typesetting
\usepackage{booktabs}       % professional-quality tables
\usepackage{amsfonts}       % blackboard math symbols
\usepackage{nicefrac}       % compact symbols for 1/2, etc.
\usepackage{microtype}      % microtypography
\usepackage{lipsum}		% Can be removed after putting your text content
\usepackage{graphicx}
\usepackage{natbib}
\usepackage{doi}
\usepackage{amsthm}
\usepackage{amsmath}
\newtheorem{assumption}{Assumption}
\usepackage{amsfonts}
\newtheorem{lemma}{Lemma}
\newtheorem{proposition}{Proposition}
\newtheorem{corollary}{Corollary}
\newtheorem{theorem}{Theorem}
\usepackage{amssymb}

\title{Fairness Under Prediction Error in Predict-Then-Optimize Pipelines}

\author{ \href{https://orcid.org/0009-0005-4873-7719}{\includegraphics[scale=0.06]{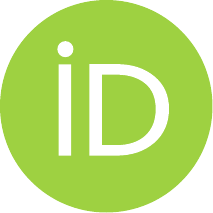}\hspace{1mm}Hilda Adwubi Osei} \\
	Department of Industrial Engineering\\
	Kwame Nkrumah University of Science and Technology\\
	  Kumasi, Ghana \\
	\texttt{hildaosei109@gmail.com} \\
}

\date{}

\renewcommand{\headeright}{Preprint}
\renewcommand{\undertitle}{Preprint}
\renewcommand{\shorttitle}{\textit{predict-then-optimize} }

\hypersetup{
pdftitle={Fairness Under Prediction Error in Predict-Then-Optimize Pipelines},
pdfsubject={cs.LG, cs.CY},
pdfauthor={Hilda A., Osei},
pdfkeywords={fairness, demographic parity, predict-then-optimize, optimization, prediction error},
}

\begin{document}
\maketitle

\begin{abstract}

	Fairness constraints in predict-then-optimize pipelines are necessarily enforced on predicted outcomes $\hat y$, since true outcomes $y$ are unobserved at decision time. We study whether a demographic parity constraint enforced this way continues to hold against $y$, and how the resulting violation scales with prediction
	error. Using a budget-constrained allocation problem solved as a linear program, we show empirically that the true-outcome fairness gap grows with prediction error $\sigma$ and saturates, well fit by
	$c(1-e^{-\sigma/\tau})$ with $c=0.1545\pm0.0103$ and $\tau=3.9045\pm0.3745$; its mean is invariant to sample size while its variance decays at the standard $O(1/\sqrt n)$ rate. We derive a
	two-part theoretical bound: a closed-form local sensitivity result
	showing the gap's first-order response to prediction error depends on
	only two individuals per group, validated against the solver to within three significant figures; and a universal, error-independent
	ceiling following from the boundedness of the decision variable, confirmed consistent with but looser than the empirical saturation
	level. On US Census income data, enforcing demographic parity on a trained predictor's output reduces an unconstrained disparity in true
	income to roughly one-third of its original size, leaving $34.4\%\pm3.3\%$ (California) and $35.7\%\pm1.8\%$ (Texas) of the
	disparity intact and undetectable from the predictions alone -- consistent not only in mechanism but in approximate magnitude across
	two states with different demographics and income distributions, and stable across ten independent train/test splits in each. A fairness constraint that appears fully satisfied at decision time can mask a substantial, systematic violation of the outcome it was meant to protect.

\end{abstract}


\section{Introduction}
\label{sec:introduction}

A prediction $\hat y$ that feeds a downstream decision is not evaluated on its own terms. In predict-then-optimize pipelines -- allocating  a budget across loan applicants, prioritizing recipients of a limited intervention, and distributing aid across a population -- a model's output is consumed by an optimizer that chooses an action before the true outcome $y$ is observed. Any fairness constraint placed on that optimizer is therefore necessarily a constraint on $\hat y$, not on $y$: the quantity a practitioner can certify as fair at decision time is not the quantity the fairness guarantee is meant to protect.

This study asks whether a demographic parity constraint enforced on $\hat y$ continues to hold when evaluated against $y$ and characterizes the gap between the two as a function of the prediction error. We study this in a budget-constrained resource-allocation problem solved as a linear program with demographic parity imposed on per-capita benefit at decision time.

Our contributions are as follows.

\textbf{(1) Empirical characterization.} Using synthetic data with controllable prediction error $\sigma$, we show that the true-outcome fairness gap grows with $\sigma$ and saturates toward a fixed ceiling rather than increasing without bound, well fit by $c \cdot (1 - e^{-\sigma/\tau})$ across more than an order of magnitude of $\sigma$ (Section~\ref{sec:empirical}). We further show that the gap's mean is approximately invariant to the sample size $n$, while its variance across draws shrinks at the standard $O(1/\sqrt{n})$ rate -- indicating a systematic bias that additional data does not correct, distinct from ordinary estimation noise.

\textbf{(2) A two-part theoretical bound.} We derive a structural result (Section~\ref{sec:theory}): at any optimum of this allocation problem, at most two individuals per group have a fractional allocation, and the gap's first-order sensitivity to prediction error depends only on these two ``marginal'' individuals -- every other individual's prediction error has zero local effect. This reduces the local sensitivity to a closed-form expression, which we validate numerically against the solver to within three significant figures. We pair this local result with a universal, $\sigma$-independent ceiling that follows directly from the boundedness of the decision variable and holds for arbitrarily large prediction error, and we show that it is consistent with, though looser than, the empirically observed saturation level.

\textbf{(3) Validation on real data.} We confirm the mechanism using a trained income predictor on US Census data (\texttt{ACSIncome}; \citet{ding2021retiring}) (Section~\ref{sec:real-data}). Enforcing demographic parity on the predictor's output reduces the unconstrained disparity in true income by approximately two-thirds but leaves the remainder intact and undetectable from the predictions alone, a result stable across ten independent train/test splits and consistent in both mechanism and approximate magnitude across two states with different demographics (California and Texas).

\textbf{Scope.} We isolate the \emph{estimation error} -- predictions drawn from the same distribution as the true outcome, corrupted by noise -- as the source of prediction imperfection. We do not study distribution shift, in which the feature-outcome relationship itself changes between training and deployment (Section~\ref{sec:discussion}). We restrict our attention to demographic parity on continuous per-capita benefit; other fairness definitions, such as equal opportunity, are not addressed here and may exhibit different behavior under the same mechanism.

Section~\ref{sec:related-work} positions this work relative to existing literature on fairness in predict-then-optimize pipelines. Section~\ref{sec:setup} formalizes the decision problem and fairness gap. Sections~\ref{sec:empirical}--\ref{sec:validation} present the empirical study, theoretical bound, and their comparison. Section~\ref{sec:real-data} presents the real-data experiment. Section~\ref{sec:discussion} discusses limitations and future work, and Section~\ref{sec:conclusion} concludes.

\section{Related Work}
\label{sec:related-work}

This work sits at the intersection of predict-then-optimize decision-making and algorithmic fairness. We organize related work by the question each line addresses and state after each how it differs from ours.

\textbf{Decision-focused learning.} \citet{elmachtoub2022smart} formalize predict-then-optimize as a learning problem and propose training predictors to minimize downstream decision regret rather than prediction error alone. \citet{wilder2019melding} extend this idea to combinatorial optimization problems, differentiating through the optimization step during training so the predictor is shaped by the decisions it will induce. Neither considers a fairness constraint on the resulting decisions; both aim to make the decision \emph{better}, not fairer. We take a trained predictor's error as given, rather than optimizing against it, and ask a narrower question: holding the predictor fixed, how does its error propagate into a fairness violation.

\textbf{Fairness in predict-then-optimize pipelines.} Closest to this work, \citet{wang2026endtoendfairnessoptimizationfair} introduced \emph{decision regret} as the gap between a fairness objective evaluated on a prediction-based decision and on the full-information (true-outcome) decision -- the same quantity we study. Their method, fair decision-focused learning, minimizes this gap during training: predictions are optimized jointly for accuracy and fairness of the decisions they induce, using a differentiable relaxation of the downstream optimization problem. This requires retraining the predictor and assumes that the fairness objective is available as a training signal. We do not propose a training method; instead, we characterize how this gap behaves as a function of prediction error \emph{alone}, for a fixed, already trained predictor, via a closed-form local sensitivity result and a universal bound requiring no retraining or differentiable relaxation to apply. Our results could determine, prior to any retraining, whether a given predictor's error already makes their training-time correction necessary.

\textbf{Fairness in budget-constrained decision-making.}
\citet{corbettdavies2017cost} study fairness constraints in threshold-based decisions under a budget (e.g.\ a limited number of interventions), and show that imposing standard fairness criteria can trade off against the decision-maker's objective -- the ``cost of fairness.'' Their analysis takes the risk score (our $\hat y$) as exact; it does not consider how error in that score affects whether the fairness criterion holds against the true outcome. Our budget-constrained allocation problem is structurally similar to theirs, but our question -- does the constraint survive contact with $y$ -- is orthogonal to the cost-of-fairness trade-off they study, which concerns $\hat y$ throughout.

\textbf{Fairness under finite-sample uncertainty.}
\citet{cousins2022uncertainty} bound how finite sample size distorts estimates of a fairness objective, where uncertainty arises from estimating group-level statistics from a limited number of individuals. This is structurally different from our setting: their uncertainty source is \emph{how many people are observed}, ours is \emph{how accurately each individual's outcome is predicted}. Consistent with this distinction, we find in Section~\ref{sec:empirical} that our fairness gap does not vanish as sample size grows, whereas sample-complexity bounds of this type generally do -- the two sources of uncertainty have different asymptotic behavior, not merely different derivations.

\textbf{Where fairness should be enforced.}
\citet{devos2025decisioncentricfairnessevaluationoptimization} argue that fairness should be measured in the region of the score distribution where decisions actually change, rather than globally. This is a question of \emph{where} to measure fairness; our question is, given a fixed, sensible place to enforce it (at decision time, on the available prediction), \emph{whether} the resulting guarantee survives contact with the true outcome. The two questions are complementary rather than competing.

\textbf{Demographic parity.} We adopt the demographic parity criterion of
\citet{dwork2012fairness}, applied here to a continuous per-capita benefit rather than a binary classification rate. We do not consider alternative criteria such as equal opportunity; Section~\ref{sec:discussion} discusses this as a direction for future work.

\textbf{Summary.} No existing work, to our knowledge, derives an explicit bound relating a true-outcome fairness violation to prediction error and sample size for a demographic-parity-constrained allocation problem, built from first principles (LP sensitivity and duality) and validated against both a controlled synthetic study and real data. This is the gap this paper addresses.

\section{Problem Setup}
\label{sec:setup}

\subsection{Notation}
\label{subsec:notation}

Table~\ref{tab:notation} summarizes the notation used throughout the
paper. Once introduced here, each symbol retains its meaning in every
subsequent section.

\begin{table}[h]
	\centering
	\caption{Notation used throughout the paper.}
	\label{tab:notation}
	\begin{tabular}{ll}
		\toprule
		Symbol & Meaning \\
		\midrule
		$n$                    & Population size \\
		$a_i \in \{0,1\}$      & Group membership of individual $i$ \\
		$G_0, G_1$             & The two groups, $G_g = \{i : a_i = g\}$ \\
		$n_0, n_1$             & Group sizes, $n_g = |G_g|$ \\
		$y_i \in \mathbb{R}$   & True outcome of individual $i$ \\
		$\hat y_i \in \mathbb{R}$ & Predicted outcome of individual $i$ \\
		$z_i \in [0,1]$        & Allocation decision for individual $i$ \\
		$B$                    & Total budget, $B = \beta n$ for budget fraction $\beta \in (0,1)$ \\
		$\sigma$               & Prediction error scale (synthetic study only) \\
		$\mathrm{gap}(v, z)$   & Fairness gap, evaluated at value vector $v$ and allocation $z$ \\
		\bottomrule
	\end{tabular}
\end{table}

\subsection{Assumptions}
\label{subsec:assumptions}

\begin{assumption}[Two groups]
	\label{asm:two-groups}
	Each individual belongs to exactly one of two groups, $G_0$ or $G_1$,
	with $n_0, n_1 \geq 1$.
\end{assumption}

\begin{assumption}[Divisible allocation]
	\label{asm:divisible}
	The decision $z_i$ is a continuous fraction in $[0,1]$, not a binary
	selection. This is standard in resource-allocation settings where the
	allocated quantity (e.g.\ funding, service hours) is itself divisible.
\end{assumption}

\begin{assumption}[Prediction available, outcome unobserved at decision time]
	\label{asm:predict-then-optimize}
	At decision time, $\hat y_i$ is known for every $i$, and $y_i$ is not.
	Any constraint imposed on the decision can therefore be a function of
	$\hat y$, but not of $y$.
\end{assumption}

\subsection{The Allocation Problem}
\label{subsec:allocation-problem}

Given predictions $\hat y \in \mathbb{R}^n$ and group labels $a \in
\{0,1\}^n$, the decision-maker solves
\begin{equation}
	\label{eq:lp}
	\begin{aligned}
		\max_{z \in \mathbb{R}^n} \quad & \hat y^\top z \\
		\text{s.t.} \quad & 0 \leq z_i \leq 1 && \forall i, \\
		& \textstyle\sum_{i=1}^n z_i \leq B, \\
		& \mathrm{gap}(\hat y, z) \leq \epsilon, \quad \mathrm{gap}(\hat y, z) \geq -\epsilon,
	\end{aligned}
\end{equation}
where the fairness constraint is written as two one-sided inequalities rather than a single equality to match the implementation used throughout the paper; the two coincide exactly when $\epsilon = 0$, which we use throughout unless stated otherwise. Problem~\eqref{eq:lp} is a linear program: the objective and every constraint are linear in $z$.

\subsection{The Fairness Gap}
\label{subsec:gap}

For any value vector $v \in \mathbb{R}^n$ (either $\hat y$ or $y$) and allocation $z$, define the per-capita benefit of group $g$ as $\frac{1}{n_g}\sum_{i \in G_g} v_i z_i$, and the fairness gap as the difference between groups:
\begin{equation}
	\label{eq:gap}
	\mathrm{gap}(v, z) \;=\; \frac{1}{n_1}\sum_{i \in G_1} v_i z_i \;-\; \frac{1}{n_0}\sum_{i \in G_0} v_i z_i.
\end{equation}
This is linear in $z$ for fixed $v$, which is what makes it usable directly as an LP constraint in~\eqref{eq:lp}.

The central object of this paper is the distinction between two evaluations of~\eqref{eq:gap} at the solution $z^*(\hat y)$ of problem~\eqref{eq:lp}:
\begin{itemize}
	\item $\mathrm{gap}(\hat y, z^*(\hat y))$ -- the \emph{enforced} gap,
	constrained to lie in $[-\epsilon, \epsilon]$ by construction, and the
	only quantity observable at decision time (Assumption~\ref{asm:predict-then-optimize});
	\item $\mathrm{gap}(y, z^*(\hat y))$ -- the \emph{true} gap: the same
	allocation $z^*(\hat y)$, evaluated against the true outcomes once observed.
\end{itemize}
These coincide only when $\hat y = y$. The remainder of the paper studies
how large $|\mathrm{gap}(y, z^*(\hat y))|$ can be as a function of the
discrepancy between $\hat y$ and $y$, given that $\mathrm{gap}(\hat y, z^*(\hat y)) \approx 0$ by construction.

\subsection{A Minimal Example}
\label{subsec:minimal-example}

We anchor the formalism with the smallest nontrivial instance: $n=2$, one
individual per group ($n_0 = n_1 = 1$), budget $B=1$, and $\epsilon = 0$. Writing $\hat y_1, \hat y_2$ for the two individuals' predictions (group~0 and group~1 respectively), the fairness constraint $\mathrm{gap}(\hat y, z) = 0$ reduces to $\hat y_1 z_1 = \hat y_2 z_2$, and combined with the budget constraint $z_1 + z_2 = 1$ (binding, since increasing either $z_i$ only helps the objective), the unique solution is
\begin{equation}
	\label{eq:toy-solution}
	z_1^* = \frac{\hat y_2}{\hat y_1 + \hat y_2}, \qquad
	z_2^* = \frac{\hat y_1}{\hat y_1 + \hat y_2}.
\end{equation}
Substituting~\eqref{eq:toy-solution} into~\eqref{eq:gap} evaluated at the
true outcomes $y_1, y_2$ gives the true gap in closed form: \begin{equation}
	\label{eq:toy-true-gap}
	\mathrm{gap}(y, z^*(\hat y)) = \frac{y_1 \hat y_2 - y_2 \hat y_1}{\hat y_1 + \hat y_2}.
\end{equation}
Equation~\eqref{eq:toy-true-gap} is zero exactly when $\hat y = y$ (up to scale), confirming that the enforced and true gaps coincide only in the absence of prediction error, and is nonzero otherwise -- the simplest possible instance of the phenomenon this paper studies at general $n$. We refer to the two individuals appearing in~\eqref{eq:toy-solution}--\eqref{eq:toy-true-gap} as \emph{marginal}: Section~\ref{sec:theory} shows that for general $n$, the solution to problem~\eqref{eq:lp} reduces to exactly this two-variable structure, restricted to at most one marginal individual per group, with every other individual's allocation pinned to $0$ or $1$.

\section{Empirical Study}
\label{sec:empirical}

\subsection{Experimental Design}
\label{subsec:emp-design}

We generate synthetic populations with two groups of controllable size and outcome gap, and a controllable amount of prediction error, so that every quantity affecting the true-outcome fairness gap is under direct experimental control. For a population of size $n$, group membership $a_i$ is drawn Bernoulli with $P(a_i=1) = 0.3$ (group~1 the minority); true outcomes are drawn $y_i \sim \mathcal{N}(\mu_{a_i}, 0.5^2)$ with $\mu_0 = 2.0$, $\mu_1 = 1.4$, so that group~1 has strictly lower expected outcome and an unconstrained allocation is unfair by construction; predictions are generated as $\hat y_i = y_i + \sigma \cdot \varepsilon_i$ for $\varepsilon_i \sim \mathcal{N}(0,1)$, so $\sigma$ is the sole, directly controlled knob governing prediction error, with $\sigma=0$ recovering perfect prediction. For each generated population, we solve problem~\eqref{eq:lp} with $\epsilon=0$ and budget fraction $\beta=0.3$, and record $\mathrm{gap}(y, z^*(\hat y))$ -- the true gap, under the allocation that was constrained to satisfy parity only on $\hat y$.

Two quantities were varied across experiments: $\sigma$ to trace how the true gap depends on the prediction error, and $n$ to determine whether the gap is a finite-sample artifact or persists asymptotically. Every other generative parameter ($\mu_0, \mu_1$, group proportion, $\beta$) was held fixed throughout, so that the observed changes in the gap could be attributed to $\sigma$ and $n$ alone. Each reported configuration was repeated over multiple random seeds; we report the mean and standard error (SE) over seeds throughout, never a single run, because any individual draw of $a, y, \varepsilon$ is itself a random quantity.

\subsection{Dependence on Prediction Error}
\label{subsec:emp-sigma-sweep}

We first sweep $\sigma \in \{0, 0.2, 0.4, \ldots, 2.0\}$ at a fixed $n=2000$, with 20 seeds per value. The mean true gap grows monotonically from $\approx 0$ at $\sigma=0$ to $-0.077$ at $\sigma=2.0$, but the
per-step increase shrinks as $\sigma$ grows (first differences fall from $\approx 0.009$ near $\sigma \approx 0.3$ to $\approx 0.006$ near $\sigma = 2.0$), indicating concave, not linear, growth. This is structurally expected: $z \in [0,1]^n$ is bounded and budget-constrained (Assumption~\ref{asm:divisible}, Eq.~\ref{eq:lp}), so $\mathrm{gap}(y,z)$ is itself bounded regardless of how large $\sigma$ becomes, ruling out unbounded (e.g.\ linear) growth on structural grounds alone.

\subsection{Dependence on Sample Size}
\label{subsec:emp-n-sweep}

Next, we sweep $n \in \{500, 1000, 2000, 4000, 8000\}$ at three fixed values of $\sigma \in \{0.5, 1.0, 1.5\}$, with 20 seeds per configuration. The mean true gap appears to \emph{decrease} in magnitude from $n=2000$ to $n=8000$ at every tested $\sigma$ (e.g.\ at $\sigma=0.5$: $-0.0204 \to -0.0180 \to -0.0154$ for $n = 2000, 4000, 8000$), suggesting the gap might vanish as $n \to \infty$. We report this
intermediate result because it did not survive closer scrutiny: the standard error at $n=500$ ($\mathrm{SE} \approx 0.006$ at $\sigma=0.5$) is nearly as large as the mean itself ($-0.0042$), indicating 20 seeds is insufficient to resolve the small-$n$ behavior, and the apparent downward trend could reflect sampling noise rather than a real effect.

Repeating with 100 seeds per configuration resolves this. The mean true gap is now approximately \emph{constant} across $n$ at every tested $\sigma$ (e.g., at $\sigma=0.5$: $-0.0131, -0.0152, -0.0154, -0.0159, -0.0154$ for $n = 500, 1000, 2000, 4000, 8000$, each within 1--2 SE of the others), while the standard deviation across seeds shrinks systematically with $n$ (from $0.0241$ at $n=500$ to $0.0061$ at $n=8000$, a $\approx 3.95\times$ reduction for a $16\times$ increase in $n$, matching the $\sqrt{16}=4$ predicted by standard sampling theory). We interpret this as two distinct components of the true gap: a systematic bias that does not vanish with $n$, and ordinary $O(1/\sqrt n)$ sampling noise around it. We return to this distinction structurally in Section~\ref{sec:theory}.

\subsection{Functional Form}
\label{subsec:emp-functional-form}

Having established that the gap's dependence on $n$ is limited to variance, not mean, we fix $n=4000$ (where the standard deviation is already small) and focus sampling effort on resolving the $\sigma$ dependence precisely, using 100 seeds per value of $\sigma \in \{0, 0.2, \ldots, 3.0\}$. We compared five candidate functional forms using weighted least squares, using each point's standard error as the weight: linear ($c\sigma$), square root ($c\sqrt\sigma$), power law ($c\sigma^p$), rational saturating ($c\sigma/(k+\sigma)$), and exponential saturating ($c(1-e^{-\sigma/\tau})$). Over this range, the square-root form is decisively rejected (RMSE $10\times$ worse than the alternatives); the power-law fit returns $p \approx 0.97$, which is indistinguishable from linear; and while the saturating forms fit measurably better, their fitted saturation parameters ($k \approx 16$, $\tau \approx 8.3$) lie well outside the tested range of $\sigma$, meaning the data over $[0,3]$ cannot yet distinguish genuine saturation from a slowly bending curve that merely resembles it.

We therefore extend the sweep to $\sigma \in \{0,1,2,3,4,6,8,10,14,20\}$, with 60 seeds per value (sufficient to resolve the shape at this coarser resolution). Over this wider range, linear growth is decisively ruled out (RMSE $30\times$ worse than the best fit), and the exponential saturating form fits best, with parameters that now fall \emph{within} the tested range ($\tau \approx 3.95$, compared to $\tau \approx 8.3$ fitted from the shorter range alone) -- indicating the earlier, short-range fit was extrapolating, while this one is interpolating. Combining both sweeps into a single weighted fit gives
\begin{equation}
	\label{eq:fitted-curve}
	|\mathrm{gap}(y, z^*(\hat y))| \approx c\left(1 - e^{-\sigma/\tau}\right),
	\qquad c = 0.1545 \pm 0.0103, \quad \tau = 3.9045 \pm 0.3745.
\end{equation}
Figure~\ref{fig:gap-vs-sigma} shows the empirical means with standard error bars alongside the fitted curve of Eq.~\eqref{eq:fitted-curve}: the gap rises steeply and nearly linearly for $\sigma \lesssim 1$, then bends over and flattens toward a ceiling of $c \approx 0.155$ by $\sigma \approx 14$--$20$, with the fitted curve passing through the error bars at every tested point.

\begin{figure}[h]
	\centering
	\includegraphics[width=0.75\textwidth]{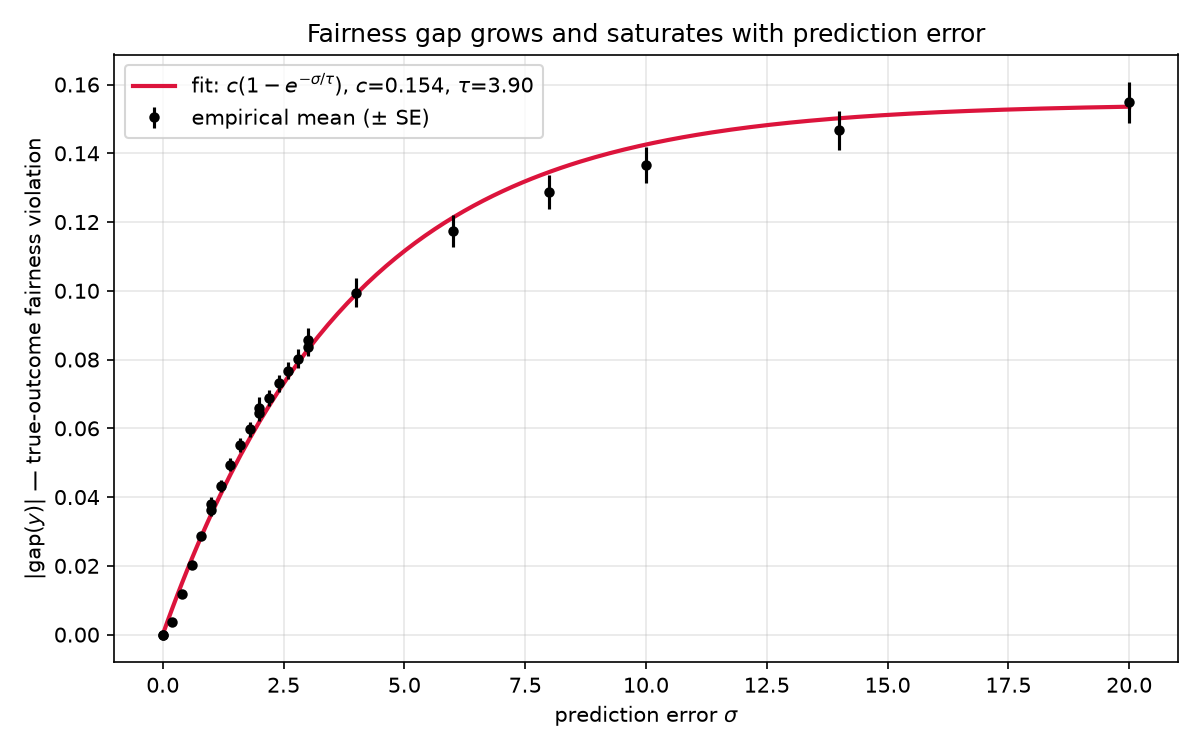}
	\caption{True-outcome fairness gap as a function of prediction error
	$\sigma$ ($n=4000$, 60--100 seeds per point, error bars show $\pm 1$
	SE), with the fitted curve of Eq.~\eqref{eq:fitted-curve}.}
	\label{fig:gap-vs-sigma}
\end{figure}

\subsection{Summary}
\label{subsec:emp-summary}

Two findings from this section are used directly in what follows. First, the true-outcome fairness gap is well described by the saturating form of Eq.~\eqref{eq:fitted-curve}, growing approximately linearly for small $\sigma$ and flattening to a finite ceiling for large $\sigma$ -- Section~\ref{sec:validation} compares this empirical shape against the theoretical bound derived in Section~\ref{sec:theory}. Second, the gap's mean is approximately invariant to sample size while its variance decays at the standard $O(1/\sqrt n)$ rate -- a signature Section~\ref{sec:theory} explains structurally by showing the gap's sensitivity to prediction error depends on only a bounded number of individuals regardless of $n$.

\section{Theoretical Analysis}
\label{sec:theory}

This section derives a theoretical characterization of the true-outcome gap in two parts: a structural result reducing the optimal allocation's dependence on $\hat y$ to at most two individuals per group (Section~\ref{subsec:thm-active-set}), from which a closed-form local sensitivity formula follows (Section~\ref{subsec:thm-local}), and a universal, $\sigma$-independent ceiling following from boundedness alone (Section~\ref{subsec:thm-ceiling}). Every derived expression in this section is validated numerically against the solver used in Section~\ref{sec:empirical}, with the validation stated immediately after each result.

\subsection{The Active Set Reduces to One Marginal Individual per Group}
\label{subsec:thm-active-set}

\begin{proposition}[Active-set reduction]
	\label{prop:active-set}
	Let $z^*$ be an optimal solution to problem~\eqref{eq:lp} with $\epsilon=0$, and suppose no two individuals within the same group share an identical predicted value $\hat y_i$. Then at most one individual in each group $G_g$ satisfies $0 < z^*_i < 1$; every other individual satisfies $z^*_i \in \{0,1\}$.
\end{proposition}

\begin{proof}
	Fix $g \in \{0,1\}$ and the total allocation to group $g$, $S_g = \sum_{i \in G_g} z_i$, at its value under $z^*$. Among all allocations within $G_g$ summing to $S_g$ with $0 \le z_i \le 1$, the one maximizing $\sum_{i \in G_g} \hat y_i z_i$ -- the only part of the objective in~\eqref{eq:lp} depending on $\{z_i\}_{i \in G_g}$ for fixed $S_g$ -- is obtained by fully funding the highest-$\hat y_i$
	individuals in $G_g$ first. Suppose instead $z^*_i, z^*_j \in (0,1)$ for two individuals $i,j \in G_g$ with $\hat y_i \neq \hat y_j$ (w.l.o.g.\ $\hat y_i > \hat y_j$). Let $\delta = \min(1-z^*_i, z^*_j) > 0$ and define $z'$ by $z'_i = z^*_i + \delta$, $z'_j = z^*_j - \delta$, $z'_k = z^*_k$ otherwise. Then $z'$ is feasible for problem~\eqref{eq:lp}: $S_g$ is unchanged, so both the budget constraint (which depends only on $\sum_i z_i$) and the fairness constraint (which depends only on $S_0, S_1$ through~\eqref{eq:gap}) are unaffected. The objective strictly increases, $\hat y^\top z' - \hat y^\top z^* = \delta(\hat y_i - \hat y_j) > 0$, contradicting optimality of $z^*$. Hence at most one individual per group is fractional.
\end{proof}

We write $m_g$ for the (at most one) marginal individual in group $g$ under $z^*$, when one exists, and refer to individuals with $z^*_i = 1$ as \emph{fully funded} and those with $z^*_i=0$ as \emph{excluded}. Proposition~\ref{prop:active-set} shows that for fixed total allocations $S_0, S_1$, the within-group structure is fully determined (fully fund by descending $\hat y_i$ until the budget for that group is exhausted, with at most one fractional remainder); the only remaining freedom is in $(S_0, S_1)$ itself, governed jointly by the budget and fairness constraints across both groups. This is exactly the two-variable system solved explicitly for $n=2$ in Section~\ref{subsec:minimal-example}, now shown to govern the general-$n$ problem as well, restricted to whichever two individuals happen to be marginal.

\begin{corollary}[Non-marginal insensitivity]
	\label{cor:non-marginal}
	For any individual $i$ that is fully funded or excluded under $z^*(\hat y)$, and for sufficiently small perturbations $\delta$ of $\hat y_i$ that do not change which individuals are fully funded, excluded, or marginal, $\partial \, \mathrm{gap}(y, z^*(\hat y)) / \partial \hat y_i = 0$.
\end{corollary}

\begin{proof}
	By Proposition~\ref{prop:active-set}, $i$'s status (fully funded or excluded) and every other individual's allocation are determined by the ranking of $\hat y$ within each group and the two marginal individuals' values, not by $\hat y_i$'s exact value, provided the perturbation is small enough that $i$'s rank relative to its group's marginal individual is unchanged. Under this condition $z^*(\hat y)$ is locally constant in $\hat y_i$, and since $\mathrm{gap}(y, \cdot)$ in~\eqref{eq:gap} is a fixed linear functional of $z$ for fixed $y$, it is locally constant in $\hat y_i$ as well.
\end{proof}

\textbf{Numerical validation.} We verified Corollary~\ref{cor:non-marginal} directly: perturbing a non-marginal individual's $\hat y_i$ by $10^{-4}$ in the $n=4000$, $\sigma=0.3$ configuration of Section~\ref{sec:empirical} changed the true gap by $\approx 3.6\times10^{-12}$, consistent with floating-point noise rather than a real effect.

\subsection{Local Sensitivity to Prediction Error}
\label{subsec:thm-local}

Write $n_{m_0}, n_{m_1}$ for the group sizes containing the marginal individuals, $F_g$ for the fully funded set in group $g$, $A_g = \frac{1}{n_g}\sum_{i \in F_g} \hat y_i$ for the fixed per-capita contribution of the fully funded individuals (as measured by $\hat y$, the quantity the constraint is enforced on), and $b_g = \hat y_{m_g}/n_g$. Let $R = B - \sum_{i \in F_0 \cup F_1} 1$ be the remaining budget for the marginal pair.

\begin{lemma}[Closed form for the marginal allocations]
	\label{lem:marginal-closed-form} Suppose both the budget and fairness constraints of problem~\eqref{eq:lp} bind at $z^*$, and both groups have a marginal individual. Then
	\begin{equation}
		\label{eq:marginal-closed-form}
		z^*_{m_0} = \frac{A_1 - A_0 + b_1 R}{b_0 + b_1}, \qquad
		z^*_{m_1} = R - z^*_{m_0}.
	\end{equation}
\end{lemma}

\begin{proof}
	With the within-group structure fixed by Proposition~\ref{prop:active-set}, the budget constraint restricted to the marginal pair is $z_{m_0} + z_{m_1} = R$, and the fairness constraint $\mathrm{gap}(\hat y, z) = 0$ becomes $A_0 + b_0 z_{m_0} = A_1 + b_1 z_{m_1}$. Substituting $z_{m_1} = R - z_{m_0}$ and solving the resulting linear equation for $z_{m_0}$ gives~\eqref{eq:marginal-closed-form}.
\end{proof}

\begin{theorem}[Local sensitivity of the true gap]
	\label{thm:local-sensitivity}
	Under the conditions of Lemma~\ref{lem:marginal-closed-form}, and for
	perturbations of $\hat y_{m_0}, \hat y_{m_1}$ small enough that the
	identity of the marginal pair and the fully-funded/excluded sets are
	unchanged,
	\begin{equation}
		\label{eq:local-sensitivity}
		\frac{\partial\, \mathrm{gap}(y,z^*)}{\partial \hat y_{m_0}}
		= \frac{z^*_{m_0}}{n_0(b_0+b_1)}\left(\frac{y_{m_0}}{n_0}+\frac{y_{m_1}}{n_1}\right),
		\qquad
		\frac{\partial\, \mathrm{gap}(y,z^*)}{\partial \hat y_{m_1}}
		= -\frac{z^*_{m_1}}{n_1(b_0+b_1)}\left(\frac{y_{m_0}}{n_0}+\frac{y_{m_1}}{n_1}\right).
	\end{equation}
\end{theorem}

\begin{proof}
	Write $N = A_1 - A_0 + b_1 R$, so $z^*_{m_0} = N/(b_0+b_1)$ by
	Lemma~\ref{lem:marginal-closed-form}, and $b_0 = \hat y_{m_0}/n_0$ is
	the only term depending on $\hat y_{m_0}$ (for fixed $\hat y_{m_1}$,
	$A_0, A_1, R$ depend on neither). Then $\partial z^*_{m_0}/\partial b_0 = -N/(b_0+b_1)^2 = -z^*_{m_0}/(b_0+b_1)$, and by the chain rule with $\partial b_0/\partial \hat y_{m_0} = 1/n_0$,
	$\partial z^*_{m_0}/\partial \hat y_{m_0} = -z^*_{m_0}/(n_0(b_0+b_1))$.
	By Corollary~\ref{cor:non-marginal} the true gap depends on $\hat
	y_{m_0}$ only through $z^*_{m_0}$ and $z^*_{m_1} = R - z^*_{m_0}$:
	differentiating $\mathrm{gap}(y,z) = \frac{1}{n_1}\sum_{G_1}y_iz_i -
	\frac{1}{n_0}\sum_{G_0}y_iz_i$ (Eq.~\ref{eq:gap}) with respect to
	$\hat y_{m_0}$ gives $\frac{y_{m_1}}{n_1}\frac{\partial z^*_{m_1}}{\partial \hat y_{m_0}} - \frac{y_{m_0}}{n_0}\frac{\partial
	z^*_{m_0}}{\partial \hat y_{m_0}}$, and substituting $\partial
	z^*_{m_1}/\partial \hat y_{m_0} = -\partial z^*_{m_0}/\partial \hat
	y_{m_0}$ gives the first expression in~\eqref{eq:local-sensitivity}.
	The second follows symmetrically, differentiating Lemma~\ref{lem:marginal-closed-form} with respect to $\hat y_{m_1}$
	through $b_1$.
\end{proof}

\textbf{Numerical validation.} We validated Theorem~\ref{thm:local-sensitivity} against the real solver at $n=4000$, $\sigma=0.3$ (seed 1), where the marginal pair had $n_0=2820$, $n_1=1180$, $z^*_{m_0}=0.632$,
$z^*_{m_1}=0.368$. Central finite differences with step $10^{-4}$--$10^{-5}$ (small enough not to disturb the active set -- a step of $10^{-3}$ was large enough to push the group-0 marginal individual to $z=1$, giving a spurious slope) gave $0.000230$ for $\partial\,\mathrm{gap}/\partial \hat y_{m_0}$ against~\eqref{eq:local-sensitivity}'s prediction of $0.000229$, and $-0.000327$ for $\partial\, \mathrm{gap}/\partial \hat y_{m_1}$ against a prediction of $-0.000318$ -- agreement to 2--3 significant figures in both cases. We note as a methodological point that an earlier derivation of Theorem~\ref{thm:local-sensitivity} carried a sign error (the two terms of Eq.~\ref{eq:gap} transposed), caught precisely by this numerical check rather than by inspection; every analytical result in this section was validated this way before being accepted.

\subsection{A Universal Ceiling}
\label{subsec:thm-ceiling}

\begin{theorem}[Universal ceiling]
	\label{thm:ceiling}
	For any feasible $z$ in problem~\eqref{eq:lp} (i.e.\ $0 \le z_i \le 1$
	for all $i$ and $\sum_i z_i \le B$) and any value vector $v$,
	\begin{equation}
		\label{eq:ceiling}
		|\mathrm{gap}(v,z)| \;\le\; \frac{B}{\min(n_0,n_1)} \max_i |v_i|.
	\end{equation}
\end{theorem}

\begin{proof}
	By the triangle inequality applied to~\eqref{eq:gap},
	$|\mathrm{gap}(v,z)| \le \frac{1}{n_1}\sum_{G_1}|v_i|z_i +
	\frac{1}{n_0}\sum_{G_0}|v_i|z_i \le \max_i|v_i| \left(
	\frac{1}{n_1}\sum_{G_1}z_i + \frac{1}{n_0}\sum_{G_0}z_i \right) \le
	\frac{\max_i|v_i|}{\min(n_0,n_1)}\sum_i z_i \le \frac{B \max_i|v_i|}{\min(n_0,n_1)}$,
	using $z_i \le 1$ and $\sum_i z_i \le B$ in the last step.
\end{proof}

Unlike Theorem~\ref{thm:local-sensitivity}, Theorem~\ref{thm:ceiling} requires no assumption on which constraints bind, which individuals are marginal, or how large the prediction error is; it holds even as $\sigma \to \infty$, when predictions carry no information about $y$. This is the sense in which it is \emph{universal}; it is correspondingly loose, as it does not use the fact that $z^*$ is chosen by an optimizer at all. In our experiments, outcomes are nonnegative ($y_i \ge 0$, by construction in the synthetic study and by the clipping step of Section~\ref{sec:real-data}), so $\max_i|y_i| = \max_i y_i$; evaluating Eq.~\eqref{eq:ceiling} at $n=4000$, $\sigma=1.0$, $\beta=0.3$ gives a ceiling of $4.04$, compared to the empirically observed saturation level of $c \approx 0.155$ from Eq.~\eqref{eq:fitted-curve} -- consistent with, though substantially looser than, the saturation Section~\ref{sec:empirical} observes. We compare the two results directly in Section~\ref{sec:validation}.

\section{Bound Validation}
\label{sec:validation}

This section compares the theoretical results of Section~\ref{sec:theory}
to the empirical curve of Section~\ref{sec:empirical}; it introduces no
new derivations.

\textbf{The universal ceiling is loose, as expected.} Theorem~\ref{thm:ceiling} provides a ceiling of $4.04$ at $n=4000$, $\sigma=1.0$, $\beta=0.3$, roughly $26\times$ the empirically observed saturation level $c \approx 0.155$ from Eq.~\eqref{eq:fitted-curve}. This looseness is expected rather than a shortcoming: Theorem~\ref{thm:ceiling} uses only the boundedness of $z$, not the fact that $z^*$ is chosen by an optimizer balancing a budget and fairness constraint against the objective. It is a worst-case bound over all feasible $z$, not a description of the $z^*$ the LP actually produces, so a substantial gap between the two is exactly what should be expected of a bound this general. Across every $(\sigma,n)$ configuration in Section~\ref{sec:empirical}, the empirical gap never approached, let alone exceeded, the corresponding ceiling -- the bound is never violated, which is the property it was constructed to guarantee.

\textbf{The local sensitivity result is validated as a mechanism, not as
a direct predictor of the population-level curve.} Theorem~\ref{thm:local-sensitivity} was validated in Section~\ref{subsec:thm-local} at the level it was derived: the sensitivity of the true gap to one marginal individual's prediction, at one specific realization of $(\hat y, a)$, confirmed against finite differences on the solver to 2--3 significant figures. This is a different quantity from the \emph{population-averaged} slope in Eq.~\eqref{eq:fitted-curve} near $\sigma = 0$, which would additionally require taking an expectation over the identity and values of the marginal pair as $\sigma$ varies across random draws, a step we do not attempt here. What Theorem~\ref{thm:local-sensitivity} and
Corollary~\ref{cor:non-marginal} do establish, and what the empirical curve is consistent with, is the \emph{mechanism}: that the gap's response to prediction error is governed by a small, bounded number of individuals rather than growing with $n$, matching the $n$-invariance of the mean gap observed in Section~\ref{subsec:emp-n-sweep}, and that this local, per-individual sensitivity is well-defined only while the active set is stable, consistent with the near-linear behavior of Eq.~\eqref{eq:fitted-curve} for small $\sigma$, before repeated swaps of the marginal pair at larger $\sigma$ average into the observed saturation.

\textbf{What this does not achieve.} Neither theoretical result derives
Eq.~\eqref{eq:fitted-curve}'s exponential saturating form, or the specific values of $c$ and $\tau$, from first principles. Doing so would require a distributional assumption on how the identity of the marginal pair changes under repeated independent draws of the noise $\varepsilon$ -- the population-level analog of the per-draw mechanism Theorem~\ref{thm:local-sensitivity} characterizes -- which we leave as an open direction (Section~\ref{sec:discussion}) rather than introduce unjustified assumptions to close it here. The two-part result of Section~\ref{sec:theory} is deliberately an envelope, not a reconstruction of the fitted curve: a provable local mechanism plus a provable universal ceiling, between which the empirically observed curve must lie, and does.

\section{Real-Data Experiment}
\label{sec:real-data}

\subsection{Data and Task}
\label{subsec:real-data-setup}

We use the \texttt{ACSIncome} task from the \texttt{folktables} package \citep{ding2021retiring} with 2018 1-year ACS PUMS vintage. In place of the task's default binary income threshold, we use the continuous reported income (\texttt{PINCP}) directly as $y$, consistent with the continuous-outcome setting of Sections~\ref{sec:setup}--\ref{sec:theory}. We define the two groups as White (\texttt{RAC1P}=1) and Black (\texttt{RAC1P}=2) respondents, the standard comparison for this dataset in the fairness literature. We run the full pipeline on two states, California and Texas, to assess whether the results are specific to one state's income distribution and demographics (Section~\ref{subsec:real-data-tx}).

For California, after removing missing income values, clipping negative income to zero, and winsorizing the top 1\% to limit the influence of extreme earners, $n_0 = 121{,}006$ and $n_1 = 8{,}557$, with group~1's mean income (\$50{,}985) below group~0's (\$70{,}039), consistent with the sign convention used throughout the paper.

Pulling \texttt{PINCP} and \texttt{RAC1P} naively from the source dataframe by the processed feature table's index produced silently misaligned income values (confirmed by the near-zero correlation between income and hours worked or education, which should correlate positively): \texttt{folktables}'s internal preprocessing resets the feature, label, and group arrays to a fresh index disconnected from the original dataframe's row labels. We corrected this by extracting income from the library's own filtering step (\texttt{ACSIncome.\_preprocess}) directly, which restored the expected positive correlations before proceeding, and confirmed the same fix is required and sufficient on Texas's data (Section~\ref{subsec:real-data-tx}).

\subsection{Predictor}
\label{subsec:real-data-predictor}

We train a random forest regressor (200 trees, max depth 10) to predict income from the task's remaining features, excluding the group attribute \texttt{RAC1P} so that any resulting prediction error reflects ordinary predictive difficulty rather than direct use of group membership. We use a 70/30 train/test split and evaluate $\hat y$ only on the held-out test set to avoid overfitting and understating the true prediction error, which is the central quantity studied in this paper. In California, one representative split achieves $R^2 = 0.407$ with negligible bias (mean residual $\approx 15$, against incomes on the order of $10^4$--$10^5$); in Texas, $R^2 = 0.417$ with similarly negligible bias (mean residual $\approx 89$). Both are realistic, non-overfit predictors, and are not particularly strong or weak for this task.

\subsection{Results: California}
\label{subsec:real-data-results}

On the representative California split above, solving problem~\eqref{eq:lp} with $\beta=0.3$ yields an unconstrained true gap of $-\$19{,}417$; enforcing the fairness constraint on $\hat y$ reduces this to $\mathrm{gap}(\hat y, z^*) \approx 0$ (as designed) while the true gap $\mathrm{gap}(y,z^*)$ is $-\$6{,}910$ -- nonzero, but substantially smaller than the unconstrained disparity.

Because any single split is itself a random quantity (Section~\ref{subsec:emp-design}), we repeated this procedure -- retraining the predictor from scratch each time -- across 10 independent train/test splits. The result was stable: unconstrained gap, $-\$19{,}112 \pm \$872$; constrained true gap, $-\$6{,}577 \pm \$738$; and, expressing the constrained gap as a fraction of the unconstrained one, $34.4\% \pm 3.3\%$ of the original disparity survives the constraint, consistently across all 10 splits (never below $29\%$ or above $40\%$). This was the headline California result; the single-split numbers above are an illustrative first pass, not independent evidence of the result.

\subsection{Cross-State Robustness: Texas}
\label{subsec:real-data-tx}

To test whether the California result reflects a general mechanism or is specific to California's income distribution and demographics, we repeat the entire pipeline unchanged -- same task, same group pair, same predictor architecture and hyperparameters, same cleaning procedure -- on Texas. Texas has a meaningfully different group balance ($n_1/n_0 \approx 11.8\%$, vs.\ California's $\approx 7.1\%$) and a different absolute income distribution (group~0 mean \$58{,}539, group~1 mean \$42{,}092, both lower than California's), providing a genuinely independent test rather than a near-duplicate of the California setting.

The 10-split result on Texas is $-\$16{,}357 \pm \$826$ (unconstrained),
$-\$5{,}846 \pm \$431$ (constrained true gap), and $35.7\% \pm 1.8\%$ of the disparity retained. Table~\ref{tab:cross-state} summarizes both states. While the absolute dollar magnitudes differ between states, consistent with their different income distributions, the \emph{fraction} of the unconstrained disparity that survives the constraint is nearly identical -- $34.4\%$ vs.\ $35.7\%$, well within each other's uncertainty -- despite the two states differing in group balance, absolute income levels, and predictor error ($R^2=0.407$ vs.\ $0.417$). This indicates the phenomenon studied in this paper is not an artifact of one state's data, and that its approximate magnitude, not only its existence, may generalize.

\begin{table}[h]
	\centering
	\caption{Real-data results, California vs.\ Texas (10-split mean $\pm$ std).}
	\label{tab:cross-state}
	\begin{tabular}{lll}
		\toprule
		 & California & Texas \\
		\midrule
		Group sizes ($n_0$, $n_1$)     & $121{,}006$, $8{,}557$   & $105{,}835$, $12{,}529$ \\
		Predictor $R^2$                & $0.407$                  & $0.417$ \\
		Unconstrained gap              & $-\$19{,}112 \pm \$872$  & $-\$16{,}357 \pm \$826$ \\
		Constrained true gap           & $-\$6{,}577 \pm \$738$   & $-\$5{,}846 \pm \$431$ \\
		Fraction of disparity retained & $34.4\% \pm 3.3\%$       & $35.7\% \pm 1.8\%$ \\
		\bottomrule
	\end{tabular}
\end{table}

\subsection{Scope of this Experiment}
\label{subsec:real-data-scope}

Two aspects of this setup remain worth noting as scope. We use a single task (income) and a single group pair (White/Black); we do not claim that the specific retained-disparity figures generalize beyond this comparison --other group pairs or outcome variables may behave differently, and two states, while a genuine robustness check, are not a representative sample of all fifty. The consistency observed here is evidence against the result being a California-specific artifact, not proof of universal generalization. We discuss this boundary further in Section~\ref{sec:discussion}.

\section{Discussion}
\label{sec:discussion}

\textbf{Fairness definition.} Every result in this study concerns demographic parity on continuous per-capita benefits (Eq.~\ref{eq:gap}); we do not study how prediction error propagates under other common fairness criteria, such as equal opportunity, which conditions on the true outcome rather than comparing unconditional group averages. Because the definition of equal opportunity already references $y$ in its standard form, the mechanism studied here -- a constraint stated purely in terms of $\hat y$ silently drifting from its true-outcome counterpart -- may not transfer directly; characterizing this for conditional fairness criteria is a natural and, we expect, non-trivial extension.

\textbf{Decision structure.} Problem~\eqref{eq:lp} is a specific linear
program: continuous allocation, a single budget constraint, a single pairwise fairness constraint. The active-set reduction (Proposition~\ref{prop:active-set}) relies on this specific structure -- in particular, on there being exactly two non-trivial constraints beyond the box constraints on $z$. Decision problems with more constraints (e.g.\ multiple budgets, more than two groups) would admit a correspondingly larger, but still bounded, active set, and we expect the qualitative conclusion (a small, fixed number of individuals governing local sensitivity, independent of $n$) to generalize; we have not verified this beyond the two-constraint case analyzed here.

\textbf{The active-set stability threshold.} Theorem~\ref{thm:local-sensitivity} holds only for perturbations small enough that the marginal pair's identity does not change (Section~\ref{subsec:thm-local}); we verified that this condition holds for specific, small finite-difference steps, but we do not derive a general formula for how large a perturbation is permissible before a rank swap occurs. Such a formula would depend on the gap between the marginal individual's value and their nearest competitor within the group, and deriving it in general would likely require a distributional assumption on the noise -- precisely the kind of assumption we deliberately avoided in favor of the looser, assumption-free envelope of Section~\ref{sec:validation}. Formalizing this threshold, and using it to derive the population-averaged small-$\sigma$ slope directly (rather than validating the per-draw mechanism only, as done here), is an open technical question.

\textbf{Estimation error versus distribution shift.} This paper isolates
\emph{estimation error}: $\hat y$ is a noisy estimate of $y$, drawn from the same underlying relationship between features and outcome. A distinct and arguably more consequential source of prediction unreliability is \emph{distribution shift}, where the feature-outcome relationship itself differs between the data a predictor was trained on and the population it is deployed on. The mechanism studied here -- sensitivity concentrated in a small active set -- was derived without reference to how $\hat y - y$ arises, and may or may not extend unchanged to shift-induced error, which need not resemble the symmetric, mean-zero noise assumed in Section~\ref{subsec:emp-design}. Extending this analysis to distribution shift is, in our view, the most consequential direction for future work.

\textbf{Practical implications.} For a practitioner auditing a fairness-constrained allocation system, the results of Sections~\ref{sec:empirical} and~\ref{sec:real-data} imply that verifying a demographic parity constraint against the predictions used at decision time is not sufficient evidence that the constraint holds in reality, even when the predictor is reasonably accurate ($R^2 \approx 0.4$ in our real-data experiment) and the verification is numerically exact. The size of the resulting true-outcome violation is not negligible or catastrophic in our experiments -- roughly a third of the unconstrained disparity survived the constraint -- but it is systematic rather than reducible by collecting more data (Section~\ref{subsec:emp-n-sweep}), meaning an audit that checks predictions alone will report a false sense of certainty regardless of how much data informed the predictor.

\section{Conclusion}
\label{sec:conclusion}

Enforcing a demographic parity constraint on predicted outcomes does not guarantee that the constraint holds against true outcomes, and the resulting violation is systematic, governed by a small, bounded number of individuals regardless of population size, not an artifact that more data resolves, a pattern we establish theoretically, validate numerically, and confirm on real income data, where roughly a third of an unconstrained disparity survives a constraint that appears, by every measure available at decision time, to have eliminated it. The most consequential extension of this result is to distribution shift, where prediction error arises not from estimation noise but from a changing relationship between features and outcomes, a setting that this paper's mechanism was not derived for, and where the size of the resulting fairness violation may behave very differently.

\section*{Code Availability}
\addcontentsline{toc}{section}{Code Availability}

All code, including the full synthetic and real-data pipelines, is
available at \url{https://github.com/hildaadwubiosei/predict-then-optimize}. The repository contains a single Jupyter notebook reproducing every result and figure reported in this paper, executed end-to-end with no manual intervention between cells. Code is released under the MIT License; written content (this paper's text and figures) is released separately under CC~BY~4.0, as noted in the repository's \texttt{LICENSE} and \texttt{LICENSE-CONTENT} files. Raw Census data used in Section~\ref{sec:real-data} is not redistributed in the repository due to its size, but is downloaded automatically on first run via the \texttt{folktables} package \citep{ding2021retiring}.

\bibliographystyle{unsrtnat}
\bibliography{references}

\end{document}